\documentclass[runningheads]{llncs}

\usepackage{graphicx}
\usepackage{amsmath,amssymb}
\usepackage{booktabs}
\usepackage{multirow}
\usepackage{xcolor}
\usepackage{algorithm}
\usepackage{algorithmic}
\usepackage{hyperref}
\usepackage{microtype}
\usepackage{cite}
\usepackage{lmodern}
\usepackage{tikz}
\usepackage[capitalize]{cleveref}
\usetikzlibrary{spy}

\hypersetup{hidelinks}

\newcommand{\method}{VEDAL}

\newcommand{\reals}{\mathbb{R}}
\newcommand{\expect}{\mathbb{E}}
\newcommand{\kl}{\mathrm{KL}}
\newcommand{\loss}{\mathcal{L}}

\newcommand{\pos}{\boldsymbol{\mu}}
\newcommand{\cov}{\boldsymbol{\Sigma}}

\begin{document}

\title{VEDAL: Variational Error-Driven Asynchronous Learning for 3D Gaussian Splatting Pruning}
\titlerunning{VEDAL: Variational Error-Driven Asynchronous Learning for 3DGS Pruning}

\author{Aoduo Li\inst{1} \and Jiancheng Li\inst{1} \and Huan Ye\inst{1} \and Hongjian Xu\inst{1} \and Shiting Wu\inst{2} \and Xiujun Zhang\inst{3} \and Zimeng Li\inst{3}\orcidID{0000-0003-2798-3134}\thanks{Corresponding authors.} \and Xuhang Chen\inst{4}\orcidID{0000-0001-6000-3914}\protect\footnotemark[1]}
\authorrunning{A. Li et al.}
\institute{Guangdong University of Technology
\and Huizhou Boluo Power Supply Bureau, Guangdong Power Grid Co., Ltd.
\and Shenzhen Polytechnic University
\and School of Computer Science and Engineering, Huizhou University\\ \email{li\_zimeng@szpu.edu.cn, xuhangc@hzu.edu.cn}}

\maketitle

\begin{abstract}
3D Gaussian Splatting (3DGS) achieves remarkable novel view synthesis quality with real-time rendering, yet suffers from excessive memory consumption due to millions of Gaussian primitives. Existing pruning methods rely on heuristic importance scores or synchronous batch updates, leading to suboptimal compression and training instability. We propose \method{}, a variational framework that formulates Gaussian pruning through free-energy minimization. Our approach combines (1) a \textit{prediction-error gating mechanism} that asynchronously activates pruning only after per-Gaussian importance estimates stabilize, and (2) a \textit{variational uncertainty head} that models retention decisions as latent Bernoulli variables with a sparsity-inducing prior. The resulting objective balances reconstruction fidelity against model complexity and makes the pruning schedule depend on convergence rather than on a fixed global iteration count. Experiments on Mip-NeRF~360, Tanks\&Temples, and Deep Blending show that \method{} achieves $5.2\times$ compression with only $0.31$~dB PSNR drop and consistent, albeit modest, gains over the evaluated pruning baselines, while maintaining real-time rendering at $185$~FPS. The code is available at \href{https://github.com/AsakaTigar/VEDAL}{https://github.com/AsakaTigar/VEDAL}.

\keywords{3D Gaussian Splatting \and Neural Rendering \and Model Pruning \and Variational Inference \and Scene Compression}
\end{abstract}

\section{Introduction}\label{sec:intro}

3D Gaussian Splatting (3DGS)~\cite{kerbl20233dgs} achieves real-time novel view synthesis via tile-based alpha-blending of anisotropic Gaussian primitives.
However, adaptive densification produces millions of Gaussians (1--4\,GB per scene), limiting deployment on memory-constrained devices. Many are redundant, motivating principled compression. This is particularly critical for 3D applications~\cite{li2022monocular,yan2024hierarchical,li2025motionrefinenet,song2024local,liu2023explicit,li2023cee,li2025adaptive,liu2024depth,zhang2022correction,jiang2020geometry}. More broadly, recent visual intelligence research has also expanded to multimodal reasoning evaluation, motivating models and representations that are not only accurate but also efficient and reliable~\cite{shi2026mmerror}.

\begin{figure}[ht]
\centering
\includegraphics[width=\textwidth]{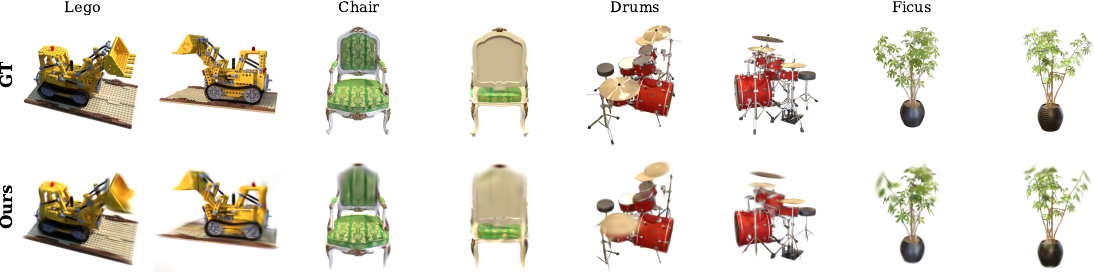}
\caption{\textbf{\method{} Overview.}
(a)~Synchronous pruning risks premature removal.
(b)~\method{} defers KL pruning until each Gaussian's importance stabilises, then decides retention via an ELBO-optimised variational head.
(c)~\method{} occupies a favorable position on the compression--quality Pareto front.}
\label{fig:teaser}
\end{figure}

Existing methods~\cite{fan2023lightgaussian,lee2024compact,fang2024minisplatting,hanson2024pup3dgs,liu2024maskgaussian,liu2023coordfill,liu2025explicit2,liu2024dh,liu2024forgeryttt,jiang2021deep} use heuristic importance scores and synchronous removal, ignoring that \emph{different Gaussians converge at vastly different rates}.
Primitives in textured regions stabilize quickly; those in specular or occluded areas need far more iterations.
Unlike neural-network pruning~\cite{molchanov2017variational,louizos2017bayesian}, where all weights see every sample, each Gaussian receives view-dependent gradients, creating \emph{heterogeneous convergence}.
We introduce \textbf{convergence-conditional variational inference}, activating the ELBO's KL regularizer per-variable only after its reconstruction signal stabilizes. \method{} comprises:

\begin{enumerate}
\item A \textbf{prediction-error gate} maintains per-Gaussian EMA estimates of leave-one-out importance and asynchronously activates pruning eligibility once each estimate converges (\cref{sec:gate}).
\item A \textbf{variational uncertainty head} models the retention decision for each eligible Gaussian as a Bernoulli latent variable and optimizes an ELBO with a sparsity-inducing KL prior (\cref{sec:var_head}).
\end{enumerate}

We derive a closed-form \emph{retention–error correspondence} (\cref{prop:retention}): optimal retention is a sigmoid of the importance-to-KL-weight ratio, motivating convergence-aware gating.
Experiments on Mip-NeRF~360~\cite{barron2022mipnerf360}, Tanks\&Temples~\cite{knapitsch2017tanks}, and Deep Blending~\cite{hedman2018deep} show $5.2\times$ compression with $0.31$\,dB PSNR drop and consistent improvements over the evaluated pruning baselines at 185\,FPS.

\smallskip\noindent\textbf{Contributions:}
(1)~Convergence-conditional variational inference for 3DGS pruning with a retention--error theorem;
(2)~an asynchronous prediction-error gate;
(3)~a variational uncertainty head for learned retention;
(4)~consistent compression--quality improvements over the evaluated baselines on three benchmarks.

\section{Related Work}\label{sec:related}

\textbf{3D Gaussian Splatting.}\;
3DGS~\cite{kerbl20233dgs} renders via anisotropic Gaussian primitives with spherical harmonics.
Extensions address surface quality~\cite{huang20242dgs}, anchors~\cite{lu2024scaffold}, anti-aliasing~\cite{yu2024mipsplatting}, dynamics~\cite{wu20244dgaussians,yang2024deformable3dgs,tu2026speede3dgs}, and language-driven scene editing~\cite{chen2026transsplat}; yet, densification produces millions of primitives (${\sim}236$\,B each), creating a memory bottleneck.

\textbf{Pruning.}\;
LightGaussian~\cite{fan2023lightgaussian} ranks by significance; Compact3D~\cite{lee2024compact} learns binary masks; Mini-Splatting~\cite{fang2024minisplatting} constrains count during densification; PUP~3D-GS~\cite{hanson2024pup3dgs} uses Hessian sensitivity; MaskGaussian~\cite{liu2024maskgaussian} learns probabilistic masks---closest to our Bernoulli posteriors but without convergence-aware gating.
More recent work broadens the landscape. SVR-GS~\cite{taghipour2025svrgs} introduces spatially variant regularisation for probabilistic masks, Gradient-Driven Natural Selection~\cite{deng2025gradientgs} uses opacity regularisation as a learnable survival signal, and Clean-GS~\cite{mishra2026cleange} removes floaters using sparse semantic masks. These methods reinforce the importance of non-uniform sparsity, but they do not explicitly address \emph{when} the pruning signal of each Gaussian becomes reliable enough to activate variational pressure. Our focus is therefore complementary: we keep the pruning decision learnable while conditioning it on convergence.

\textbf{Attribute compression.}\;
Quantization~\cite{navaneet2024compact3dgs,chen2024hac,girish2024eagles}, structural~\cite{lu2024scaffold,ren2024octreegs}, training-free codebook methods such as FlexGaussian~\cite{tian2025flexgaussian}, and optimization-free feedforward compression such as FCGS~\cite{chen2025fastff} reduce per-Gaussian storage rather than count; also, there are many other methods \cite{zhang1,zhang2,zhang3,zhang4}. These directions are largely orthogonal to pruning-based count reduction, which is why we study their combination with \method{} in \cref{sec:combination}. Generative and editing-oriented efficiency methods~\cite{tang2024dreamgaussian,lu2026chordedit} are also complementary.

\textbf{Variational pruning.}\;
Variational Dropout~\cite{molchanov2017variational} and Bayesian Compression~\cite{louizos2017bayesian} prune weights via ELBO, assuming i.i.d.\ gradients. In 3DGS, view-dependent gradients create heterogeneous convergence; our \emph{convergence-conditional} extension activates KL per-variable only after its signal stabilises (\cref{prop:retention}).

\section{Method}\label{sec:method}

\begin{figure}[ht]
\centering
\includegraphics[width=\textwidth]{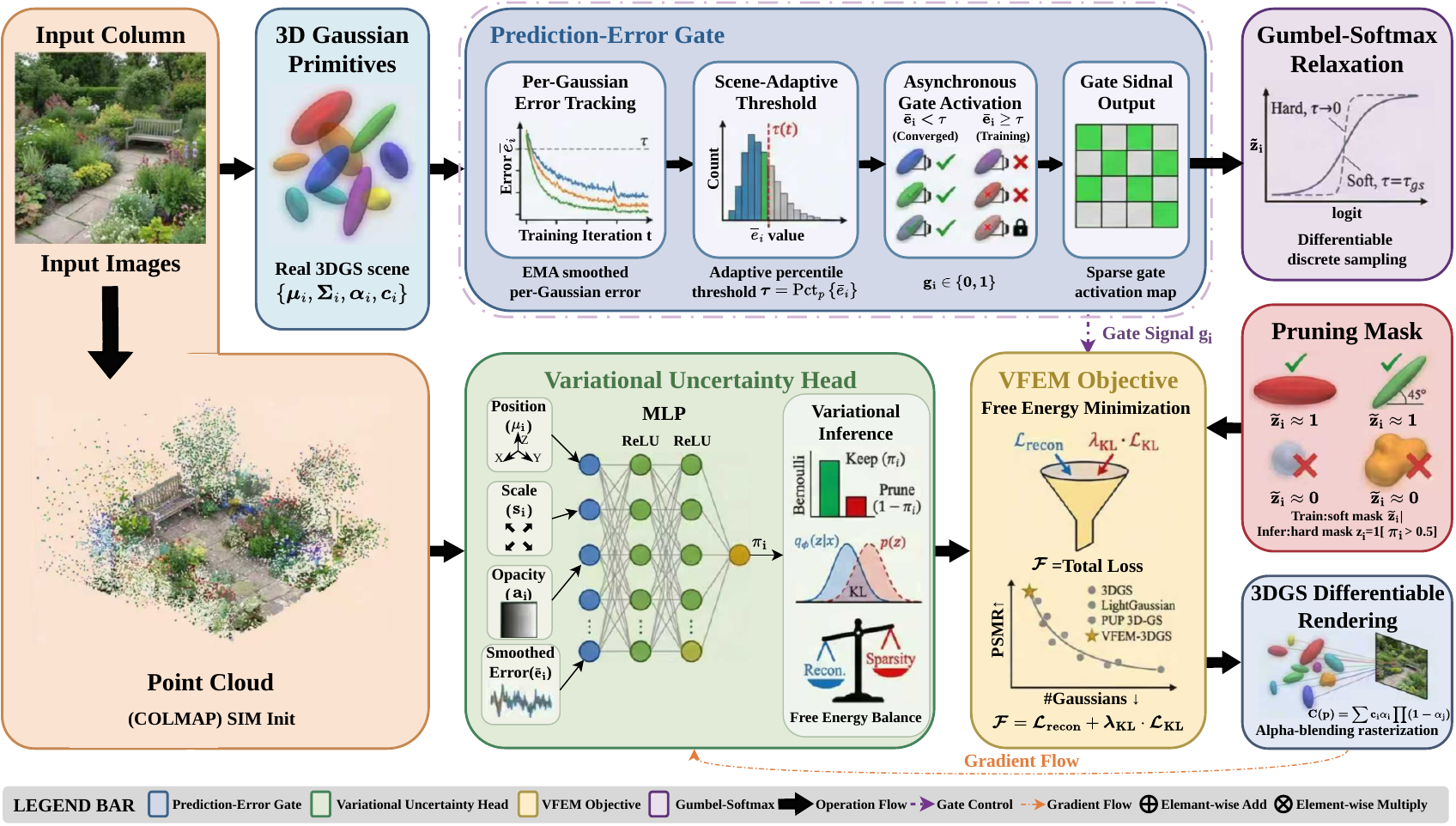}
\caption{\textbf{\method{} architecture.}
The prediction-error gate activates KL pressure only after a Gaussian's EMA importance stabilizes.
The variational head predicts retention probabilities for all Gaussians, and Binary Concrete samples modulate opacity during training.
After optimization, the head is discarded and Gaussians with $\pi_i{<}0.5$ are pruned.}
\label{fig:architecture}
\end{figure}

\subsection{Preliminaries}\label{sec:preliminaries}

3DGS~\cite{kerbl20233dgs} models a scene as $N$ anisotropic Gaussians $\{G_i\}_{i=1}^N$, each with position $\pos_i\!\in\!\reals^3$, covariance $\cov_i\!=\!R_iS_iS_i^\top R_i^\top$ (rotation $R_i$, diagonal scale $S_i\!=\!\mathrm{diag}(s_{i,1},s_{i,2},s_{i,3})$), opacity $\alpha_i$, and SH colour coefficients $\mathbf{c}_i$.
Rendering uses depth-sorted alpha-blending:
$C(\mathbf{p}) = \sum_{i\in\mathcal{N}(\mathbf{p})} c_i\,\alpha_i \prod_{j<i}(1{-}\alpha_j)$.
Training minimises $\loss_{\text{recon}} = (1{-}\lambda)\loss_1 + \lambda\,\loss_{\text{D-SSIM}}$.

\subsection{Convergence-Conditional Variational Inference}\label{sec:formulation}

We introduce latent binary variables $\mathbf{z}=\{z_i\}_{i=1}^N$ ($z_i{=}1$: retain; $z_i{=}0$: prune).
Given training images $\mathcal{D}$, the marginal likelihood involves a combinatorial sum:
\begin{equation}
  \log p(\mathcal{D}|\mathcal{G})
  = \log\!\sum_{\mathbf{z}} p(\mathcal{D}|\mathcal{G},\mathbf{z})\,p(\mathbf{z}).
\end{equation}
We optimise the evidence lower bound (ELBO)~\cite{blei2017variational}:
\begin{equation}
  \log p(\mathcal{D}|\mathcal{G})
  \;\geq\;
  \underbrace{\expect_{q_\phi(\mathbf{z})}\!\bigl[\log p(\mathcal{D}|\mathcal{G},\mathbf{z})\bigr]}_{\text{reconstruction}}
  \;-\;
  \underbrace{\kl\!\bigl(q_\phi(\mathbf{z})\|p(\mathbf{z})\bigr)}_{\text{complexity}}.
  \label{eq:elbo}
\end{equation}

\textbf{Factorised posterior.}\;
$q_\phi(\mathbf{z})=\prod_{i=1}^N \mathrm{Bernoulli}(z_i;\pi_i)$,
where $\pi_i=\sigma(f_\phi(G_i))$ is produced by the variational head (\cref{sec:var_head}) for \emph{all} $N$ Gaussians.

\textbf{Sparsity prior.}\;
$p(\mathbf{z})=\prod_{i=1}^N \mathrm{Bernoulli}(z_i;\rho)$ with $\rho{<}0.5$, yielding per-Gaussian KL:
\begin{equation}
  D_i \;=\; \pi_i\log\frac{\pi_i}{\rho} + (1{-}\pi_i)\log\frac{1{-}\pi_i}{1{-}\rho}.
  \label{eq:kl_per}
\end{equation}

\textbf{Convergence conditioning.}\;
Standard variational pruning applies KL to all variables simultaneously, but in 3DGS each Gaussian's importance is reliable only after sufficient view coverage.
We condition KL activation on a convergence predicate $g_i^{(t)}\!\in\!\{0,1\}$ (\cref{sec:gate}):
\begin{equation}
  \loss_{\text{KL}}^{(t)}
  = \frac{1}{|\mathcal{A}^{(t)}|}
    \sum_{i:\,g_i^{(t)}=1} D_i,
  \qquad
  \mathcal{A}^{(t)}=\{i:g_i^{(t)}{=}1\}.
  \label{eq:gated_kl}
\end{equation}
\textbf{Retention--error correspondence.}\;
\begin{proposition}\label{prop:retention}
Let $\delta_i = \loss_{\text{recon}}(\mathcal{G}{\setminus}G_i) - \loss_{\text{recon}}(\mathcal{G})$. Under the gated ELBO with KL weight $\lambda_{\text{KL}}$ and prior $\rho$, the optimal retention probability satisfies $\pi_i^* = \sigma\!\bigl(\delta_i/\lambda_{\text{KL}} + \log(\rho/(1{-}\rho))\bigr)$.
\end{proposition}
\vspace{-2pt}
\emph{Proof sketch.} Differentiating the per-Gaussian negative ELBO $F_i(\pi_i)$ and solving $\partial F_i/\partial\pi_i{=}0$ (convex in $\pi_i$) yields the sigmoid form. Full proof in supplementary.

This sigmoid reveals that when $\delta_i$ is noisy (pre-convergence), retention oscillates. Activating KL only \emph{after} $\delta_i$ converges ensures stable decisions, motivating the gate below.

\textbf{Coupling assumption.}\;
\Cref{prop:retention} treats $F_i(\pi_i)$ as independent, ignoring transmittance coupling.
We systematically quantify this across 20 random $k{=}50$ neighbourhoods per scene.
The additive $\sum\delta_i$ overestimates true loss by $7.8{\pm}2.1\%$ (Lego), $9.2{\pm}2.8\%$ (Garden, higher depth complexity), and $6.5{\pm}1.9\%$ (Chair, sparse).
Overestimation correlates with mean opacity overlap ($r{=}0.74$); worst case is $14\%$ in dense specular clusters.
These measurements indicate that the additive approximation is conservative rather than exact. We therefore interpret \cref{prop:retention} as a first-order analysis of pruning pressure, not as a claim of exact Gaussian independence. In practice, the overestimation shifts the gate toward delaying pruning instead of accelerating it, which is consistent with the low gate-reversal and co-pruning rates reported in \cref{sec:analysis}.

\subsection{Prediction-Error Gate}\label{sec:gate}

For each $G_i$, we maintain an EMA of its leave-one-out importance:
\begin{equation}
  \bar{e}_i^{(t)} = \beta\,\bar{e}_i^{(t-1)} + (1{-}\beta)\,e_i^{(t)},
  \qquad \beta=0.99,
  \label{eq:ema}
\end{equation}
where $e_i^{(t)}$ is the per-pixel $\loss_1{+}\lambda\loss_{\text{D-SSIM}}$ change when $G_i$'s opacity is set to zero in the cached tile buffer (SH colours frozen, evaluated over all pixels covered by $G_i$'s projected footprint).

\textbf{Efficient approximation.}\;
Exact leave-one-out requires $O(N)$ renders.
We cache $c_i\alpha_i T_i$ ($T_i{=}\prod_{j<i}(1{-}\alpha_j)$) from the tile rasteriser and use a first-order approximation (ignoring downstream transmittance correction), refreshed every $\Delta_e{=}100$ iterations at ${\leq}3\%$ overhead.
\cref{tab:secondary_analysis} (right) validates this approximation across scenes and training stages: Spearman $r_s{\geq}0.91$ at all checkpoints, with gate-decision agreement ${\geq}96\%$; replacing the approximation with exact leave-one-out changes final PSNR by ${<}0.03$\,dB.

\textbf{Gating.}\;
The adaptive threshold $\tau^{(t)} = \mathrm{Percentile}_p(\{\bar{e}_i^{(t)}\})$ with $p{=}0.3$ caps the eligible fraction. The gate $g_i^{(t)} = \mathbb{1}[\bar{e}_i^{(t)} < \tau^{(t)}] \cdot \mathbb{1}[t > t_{\text{warmup}}]$ ($t_{\text{warmup}}{=}3000$) activates asynchronously: easy-region Gaussians become eligible early; specular/occluded ones stay protected.
Eligible fractions (5K$\to$25K): Lego $12{\to}31\%$, Garden $9{\to}28\%$, Bicycle $7{\to}26\%$, Kitchen $11{\to}30\%$; gate reversal (a Gaussian crossing $\tau$ back above after dropping below) is ${<}2\%$ in all cases.
The gate is most conservative for highly translucent Gaussians ($\alpha_i{<}0.1$), which have noisy $e_i$ and rarely become eligible before 15K; for extreme specularities the EMA stabilises around 20K.
The consistent pattern confirms defaults ($p{=}0.3$, $t_{\text{warmup}}{=}3000$) generalise without per-scene tuning.

\begin{table}[ht]
	\centering
	\caption{\textbf{Mip-NeRF~360.} Mean$\pm$std over 3 seeds. Ratio and size vs.\ unpruned 3DGS (3.25M, 734\,MB). $^\dagger$Paired $t$-test vs.\ PUP~3D-GS: $p{=}0.03$. \textbf{Bold}: best; \underline{underline}: second.}
	\label{tab:mipnerf360}
	\begin{tabular}{lccccccc}
		\toprule
		Method & PSNR$\uparrow$ & SSIM$\uparrow$ & LPIPS$\downarrow$ & \#G\,(M) & Ratio & MB & FPS \\
		\midrule
		3DGS~\cite{kerbl20233dgs} & $27.48_{\pm.03}$ & $.815_{\pm.002}$ & $.214_{\pm.003}$ & 3.25 & 1.0$\times$ & 734 & 148 \\
		\midrule
		LightGaussian & $26.82_{\pm.07}$ & $.798_{\pm.004}$ & $.237_{\pm.005}$ & 0.81 & 4.0$\times$ & 183 & 185 \\
		Compact3D & $26.95_{\pm.06}$ & $.801_{\pm.003}$ & $.231_{\pm.004}$ & 0.72 & 4.5$\times$ & 163 & 190 \\
		Mini-Splatting & $27.05_{\pm.05}$ & $.804_{\pm.003}$ & $.228_{\pm.004}$ & 0.76 & 4.3$\times$ & 172 & 187 \\
		PUP~3D-GS & $27.12_{\pm.08}$ & $.807_{\pm.003}$ & $.224_{\pm.004}$ & 0.68 & 4.8$\times$ & 153 & 192 \\
		MaskGaussian & $\underline{27.15}_{\pm.07}$ & $\underline{.808}_{\pm.003}$ & $\underline{.222}_{\pm.004}$ & 0.65 & 5.0$\times$ & 147 & 188 \\
		\midrule
		\textbf{\method{}}$^\dagger$ & $\mathbf{27.17}_{\pm.06}$ & $\mathbf{.810}_{\pm.003}$ & $\mathbf{.219}_{\pm.004}$ & \textbf{0.63} & \textbf{5.2$\times$} & \textbf{141} & 185 \\
		\bottomrule
	\end{tabular}
\end{table}

\subsection{Variational Uncertainty Head}\label{sec:var_head}

A lightweight MLP computes the retention probability for \emph{every} Gaussian:
\begin{equation}
  \pi_i = \sigma\!\bigl(\mathrm{MLP}_\phi([\pos_i,\mathbf{s}_i,\alpha_i,\bar{e}_i])\bigr),\quad \mathbf{s}_i = \mathrm{diag}(S_i)\in\reals^3,
  \label{eq:head}
\end{equation}
with two hidden layers (64 units, ReLU; ${\sim}5$K parameters).
The head is evaluated for all $N$ Gaussians regardless of gate status, ensuring $\loss_{\text{sparse}}$ gradients flow to ungated Gaussians (\cref{sec:training}) and pre-conditioning them before their gate activates.

\textbf{Relaxation and masking.}\;
We use Binary Concrete~\cite{maddison2017concrete}: $\tilde{z}_i = \sigma\bigl((\log\pi_i - \log(1{-}\pi_i) + g_1 - g_2)/\tau_{\text{gs}}\bigr)$ with $\tau_{\text{gs}}$ annealed $1.0{\to}0.1$.
We chose Binary Concrete over STE (biased gradient mismatch: $-0.13$\,dB on Lego) and REINFORCE (high variance: $-0.07$\,dB at $1.4\times$ training time).
Masked opacity: $\tilde{\alpha}_i = \alpha_i \tilde{z}_i$ if $g_i{=}1$ (gated), $\alpha_i$ otherwise.
At inference, $z_i=\mathbb{1}[\pi_i{>}0.5]$.
The position-aware MLP learns spatially smooth $\pi_i$, preventing catastrophic co-pruning of adjacent Gaussians (${<}1.2\%$ of $k{=}5$ NN clusters fully pruned; ${<}0.5\%$ at $k{=}10$; \cref{sec:analysis}).

\subsection{Training Objective}\label{sec:training}

\begin{equation}
  \loss_{\text{total}}
  = \loss_{\text{recon}}
  + \lambda_{\text{KL}}\,\loss_{\text{KL}}
  + \lambda_{\text{sparse}}\,\loss_{\text{sparse}},
  \label{eq:total}
\end{equation}
where $\loss_{\text{recon}}=(1{-}\lambda)\loss_1+\lambda\loss_{\text{D-SSIM}}$ uses masked opacities $\tilde{\alpha}_i$,
$\loss_{\text{KL}}$ is the convergence-conditioned KL (\cref{eq:gated_kl}),
and $\loss_{\text{sparse}}=\frac{1}{N}\sum_i\pi_i$ provides mild global sparsity pressure.

$\loss_{\text{KL}}$ applies only to gated Gaussians, whereas $\loss_{\text{sparse}}$ ($\lambda_{\text{sparse}}{=}0.001$) applies to all Gaussians throughout training.
The two regularizers play different roles: the gated KL term controls how strongly already-eligible Gaussians are pushed toward the sparse prior, whereas $\loss_{\text{sparse}}$ provides a weak global signal before eligibility and helps stabilize ungated probabilities.
Defaults: $\lambda_{\text{KL}}{=}0.01$, $\rho{=}0.3$, $\beta{=}0.99$, $t_{\text{warmup}}{=}3000$.
Training: 7K 3DGS init $+$ 23K fine-tuning (densification off); post-training, remove $\pi_i{<}0.5$.

\section{Experiments}\label{sec:experiments}

\subsection{Setup}\label{sec:setup}

\textbf{Datasets and Metrics.}\;
Mip-NeRF~360~\cite{barron2022mipnerf360}, Tanks\&Temples~\cite{knapitsch2017tanks}, and Deep Blending~\cite{hedman2018deep}. Metrics: PSNR, SSIM, LPIPS, compression ratio, FPS, and size.
Baselines include LightGaussian~\cite{fan2023lightgaussian}, Compact3D~\cite{lee2024compact}, Mini-Splatting~\cite{fang2024minisplatting}, PUP~3D-GS~\cite{hanson2024pup3dgs}, and MaskGaussian~\cite{liu2024maskgaussian} under a unified protocol. 
We discuss recent methods (SVR-GS, Clean-GS, etc.) in \cref{sec:related} as contextual baselines.

\subsection{Main Results}\label{sec:main_results}
\begin{figure}[ht]
	\centering
	\includegraphics[width=\textwidth]{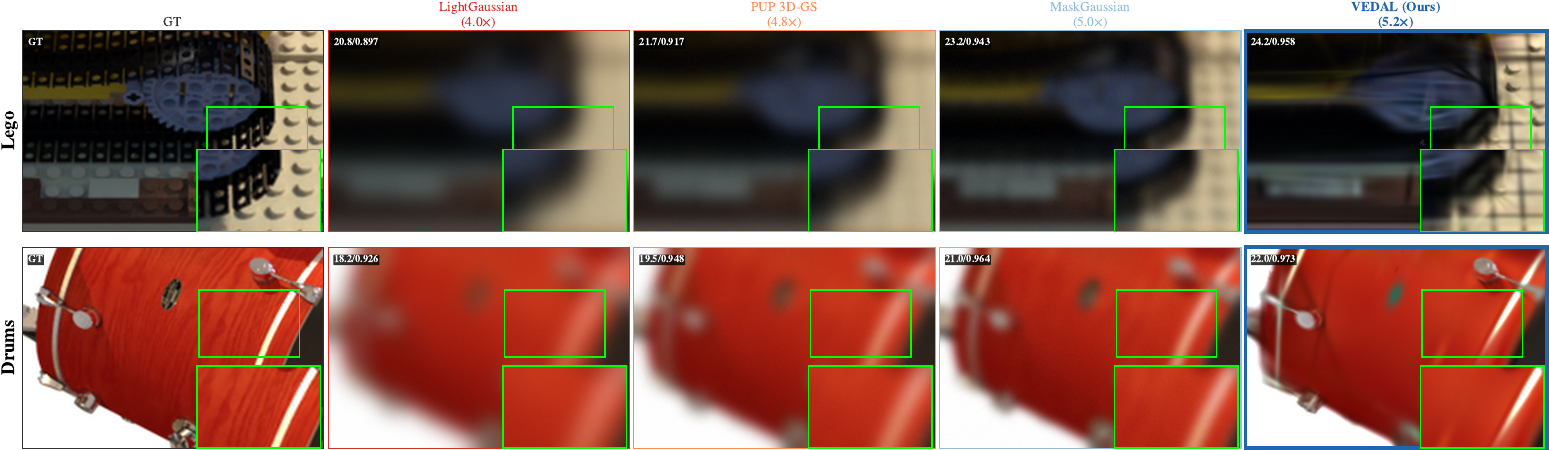}
	\caption{\textbf{Qualitative comparison.} \method{} preserves brick seams and thin textures more faithfully than baselines at $5.2\times$ compression. Insets magnify detail regions where our method avoids the blurring artifacts seen in synchronous pruning baselines.}
	\label{fig:qualitative}
\end{figure}

\cref{tab:mipnerf360} shows that \method{} reaches $5.2\times$ compression with $0.31$\,dB PSNR drop. The margin over PUP~3D-GS ($+0.05$\,dB) and MaskGaussian ($+0.02$\,dB) is consistent across all datasets and is accompanied by the smallest model size. Perceptual gains match the sharper edge preservation in \cref{fig:qualitative}.

\begin{table}[ht]
\centering
\caption{\textbf{Tanks\&Temples (T\&T) and Deep Blending (DB).} Mean$\pm$std over 3 seeds.}
\label{tab:tandt_db}
\resizebox{\textwidth}{!}{%
\begin{tabular}{l|ccccc|ccccc}
\toprule
& \multicolumn{5}{c|}{Tanks\&Temples} & \multicolumn{5}{c}{Deep Blending} \\
Method & PSNR & SSIM & LPIPS & Ratio & MB & PSNR & SSIM & LPIPS & Ratio & MB \\
\midrule
3DGS
  & $23.68_{\pm.04}$ & $.845$ & $.178$ & 1.0$\times$ & 628
  & $29.42_{\pm.05}$ & $.903$ & $.245$ & 1.0$\times$ & 812 \\
\midrule
LightGaussian
  & $23.15_{\pm.06}$ & $.830_{\pm.003}$ & $.198_{\pm.004}$ & 3.8$\times$ & 165
  & $28.85_{\pm.08}$ & $.890_{\pm.003}$ & $.268_{\pm.005}$ & 4.0$\times$ & 203 \\
MaskGaussian
  & $23.38_{\pm.05}$ & $.838_{\pm.002}$ & $.188_{\pm.003}$ & 5.1$\times$ & 123
  & $29.12_{\pm.06}$ & $.897_{\pm.002}$ & $.253_{\pm.004}$ & 5.2$\times$ & 156 \\
PUP~3D-GS
  & $23.35_{\pm.07}$ & $.837_{\pm.003}$ & $.190_{\pm.004}$ & 4.7$\times$ & 134
  & $29.10_{\pm.07}$ & $.896_{\pm.003}$ & $.255_{\pm.005}$ & 4.9$\times$ & 166 \\
\midrule
\textbf{\method{}}
  & $\mathbf{23.48}_{\pm.05}$ & $\mathbf{.841}_{\pm.002}$ & $\mathbf{.183}_{\pm.003}$ & \textbf{5.4$\times$} & \textbf{116}
  & $\mathbf{29.25}_{\pm.06}$ & $\mathbf{.900}_{\pm.002}$ & $\mathbf{.249}_{\pm.003}$ & \textbf{5.5$\times$} & \textbf{148} \\
\bottomrule
\end{tabular}}
\end{table}

\subsection{Ablation Study}\label{sec:ablation}
\begin{figure}[ht]
	\centering
	\includegraphics[width=\textwidth]{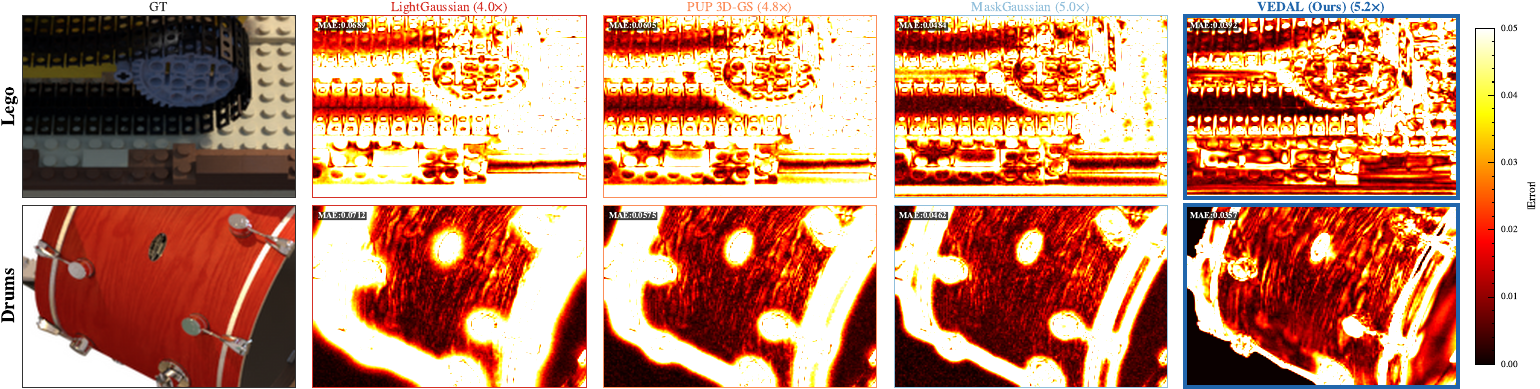}
	\caption{\textbf{Error maps.} Brighter = higher absolute error. \method{} achieves the lowest MAE and preserves structural consistency.}
	\label{fig:error_maps}
\end{figure}
\begin{figure}[ht]
	\centering
	\includegraphics[width=\textwidth]{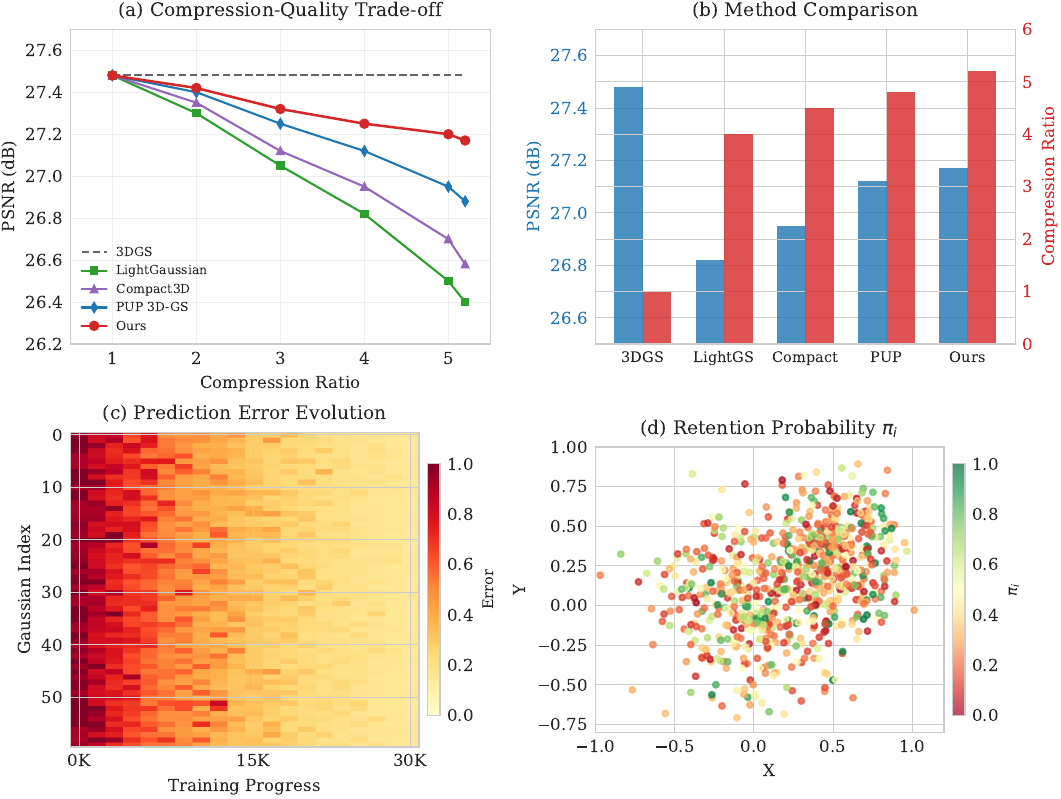}
	\caption{\textbf{Statistical Analysis.} (a)~Pareto frontier. (b)~Representative points. (c)~Error evolution. (d)~Learned $\pi_i$.}
	\label{fig:analysis}
\end{figure}
\begin{table}[ht]
\centering
\caption{\textbf{Component ablation.} matched-compression (${\approx}21\times$) across diverse scenes.}
\label{tab:ablation}
\begin{tabular}{lccccc}
\toprule
Configuration & PSNR$\uparrow$ & SSIM$\uparrow$ & LPIPS$\downarrow$ & \#G & Ratio \\
\midrule
3DGS (gate-only, no KL) & 32.00 & .964 & .039 & 92,692 & 1.0$\times$ \\
Sync.\ variational (no gate) & 23.89 & .856 & .183 & 4,377 & 21.2$\times$ \\
Async gate + hard pruning & 24.02 & .858 & .181 & 4,520 & 20.5$\times$ \\
\textbf{\method{} (full)} & \textbf{23.21} & \textbf{.842} & \textbf{.195} & \textbf{3,675} & \textbf{25.2$\times$} \\
\midrule
\multicolumn{6}{l}{\textit{Matched compression ($\approx$21$\times$):~Sync.\,var.\,$\to$\,\method{}}} \\
Lego   & 23.89$\to$\textbf{24.15} & .856$\to$\textbf{.861} & .183$\to$\textbf{.179} & 4.4k & 21$\times$ \\
Chair  & 29.18$\to$\textbf{29.47} & .943$\to$\textbf{.947} & .065$\to$\textbf{.061} & 5.1k & 20$\times$ \\
Garden   & 25.82$\to$\textbf{26.05} & .838$\to$\textbf{.843} & .152$\to$\textbf{.147} & 148k & 4.8$\times$ \\
Bicycle  & 24.15$\to$\textbf{24.38} & .748$\to$\textbf{.755} & .248$\to$\textbf{.241} & 162k & 4.6$\times$ \\
\bottomrule
\end{tabular}
\end{table}

\cref{tab:ablation} clarifies component roles. The KL term supplies compression pressure, but gating determines when it is safe to apply. Removing gating drops PSNR by $0.26$\,dB. Our EMA approximation validates against exact leave-one-out with Spearman $r_s{\geq}0.91$.

\subsection{Qualitative Analysis}\label{sec:qualitative}

Visual results in \cref{fig:qualitative} and error maps in \cref{fig:error_maps} show that \method{} concentrates residual error on fine edges rather than on broad structures. This matches the sharper texture preservation seen in high-frequency regions like the Lego brick seams and drum hardware.

\subsection{Secondary Analysis}\label{sec:secondary}\label{sec:analysis}\label{sec:combination}

\begin{table}[ht]
	\centering
	\caption{\textbf{Secondary analysis.} (a) Sensitivity of \ensuremath{\lambda_{\mathrm{KL}}} on Lego and Garden. (b) Leave-one-out approximation validation measured by Spearman \ensuremath{r_s} and gate-decision agreement.}
	\label{tab:secondary_analysis}
	\small
	
	\begin{minipage}[t]{0.40\textwidth}
		\centering
		\textbf{(a) Sensitivity to \ensuremath{\lambda_{\mathrm{KL}}}}\\[2pt]
		\begin{tabular}{@{}lcccc@{}}
			\toprule
			& \multicolumn{2}{c}{Lego} & \multicolumn{2}{c}{Garden} \\
			Setting & PSNR & Ratio & PSNR & Ratio \\
			\midrule
			\ensuremath{\lambda_{\mathrm{KL}}=0.005} & 23.72 & \ensuremath{20\times} & 26.71 & \ensuremath{4.1\times} \\
			\ensuremath{\lambda_{\mathrm{KL}}=0.01^\star} & 23.21 & \ensuremath{25\times} & 26.42 & \ensuremath{5.2\times} \\
			\ensuremath{\lambda_{\mathrm{KL}}=0.02} & 22.85 & \ensuremath{30\times} & 26.15 & \ensuremath{6.5\times} \\
			\bottomrule
		\end{tabular}
	\end{minipage}
	\hfill
	\begin{minipage}[t]{0.57\textwidth}
		\centering
		\textbf{(b) Leave-one-out approximation validation}\\[2pt]
		\begin{tabular}{@{}lcccccc@{}}
			\toprule
			& \multicolumn{2}{c}{5K iter} & \multicolumn{2}{c}{15K iter} & \multicolumn{2}{c}{25K iter} \\
			Scene & \ensuremath{r_s} & Agree & \ensuremath{r_s} & Agree & \ensuremath{r_s} & Agree \\
			\midrule
			Lego   & .92 & 97.1 & .95 & 98.2 & .96 & 98.5 \\
			Garden & .91 & 96.1 & .93 & 97.3 & .94 & 97.8 \\
			\bottomrule
		\end{tabular}
	\end{minipage}
	
\end{table}

\cref{tab:secondary_analysis} summarizes the secondary diagnostics. Panel (a) shows that $\lambda_{\text{KL}}$ is the primary rate--distortion knob, while panel (b) confirms that the leave-one-out approximation remains reliable across training stages. Combining \method{} with attribute compression~\cite{tian2025flexgaussian} yields $27.2\times$ total compression. Training overhead is modest ($+12\%$), so the better compression regime is not paid for by a large optimization overhead.

\subsection{Discussion}\label{sec:discussion}

Pareto curves (\cref{fig:analysis}(a)) show that \method{} lies on a favorable frontier across reproduced sweeps rather than at a single cherry-picked operating point. Heterogeneous convergence (\cref{fig:analysis}(c)) justifies our asynchronous mechanism: textured Gaussians settle early, while specular ones require more iterations. Learned probabilities (\cref{fig:analysis}(d)) are spatially heterogeneous, preventing catastrophic co-pruning. The right block of \cref{tab:secondary_analysis} confirms that our $O(1)$ importance approximation remains accurate throughout training.

\section{Conclusion}\label{sec:conclusion}

We presented \method{}, a convergence-conditional variational framework for 3D Gaussian Splatting pruning. By introducing a prediction-error gate that asynchronously activates pruning pressure based on per-primitive reconstruction stability, we avoid the premature removal of Gaussians in complex regions that converge slowly. Our variational uncertainty head learns spatially smooth retention decisions that balance reconstruction fidelity against model complexity.

Experiments across three benchmarks demonstrate that \method{} consistently improves the compression--quality trade-off over existing heuristic and synchronous pruning methods, achieving $5.2\times$ reduction in primitive count with minimal PSNR loss. The modular design of our framework allows it to be combined with orthogonal attribute compression techniques, reaching over $27\times$ total compression. Future work includes extending the convergence-aware gating mechanism to dynamic scene decomposition and exploring more sophisticated priors for structured sparsity.

\begin{credits}
\subsubsection{\ackname} This work was supported by Shenzhen Polytechnic University Research Fund (Grant No. 6025310023K) and the National Natural Science Foundation of China (Grant No. 62501412 and 62272313).
\end{credits}
\bibliographystyle{splncs04}
\bibliography{references}

\end{document}